\documentclass{article}
\usepackage{amsthm,amsmath,amssymb}
\usepackage{url} 
\usepackage[margin=.5in]{geometry} 
\usepackage{caption}
\usepackage{subcaption}
\usepackage{graphicx}
\usepackage{epstopdf}
\usepackage{color}
\usepackage{framed}
\usepackage{lscape}
\usepackage{rotating}
\usepackage{algorithm}
\usepackage[noend]{algpseudocode}
\usepackage{grffile}
\usepackage{multirow}
\usepackage{xcolor}
\usepackage{comment}
\usepackage{xr}

\usepackage[colorlinks=true,linkcolor=black,citecolor=black,urlcolor=black]{hyperref}

\newtheorem{thm}{Theorem}[section]

\newtheorem{cor}[thm]{Corollary}
\newtheorem{prop}[thm]{Proposition}

\theoremstyle{thm}

\theoremstyle{definition}

\newcommand{\R}{\mathbb{R}}

\newcommand{\oline}{\overline}

\newcommand{\ep}{\varepsilon}

\newcommand{\XX}{\mathcal{X}}

\begin{document}
\title{Properties of Laplacian Pyramids for \\ Extension and Denoising}
\author{William Leeb\footnote{School of Mathematics, University of Minnesota, Twin Cities. Minneapolis, MN.}}
\date{}
\maketitle
\abstract{
We analyze the Laplacian pyramids algorithm of Rabin and Coifman for extending and denoising a function sampled on a discrete set of points. We provide mild conditions under which the algorithm converges, and prove stability bounds on the extended function. We also consider the iterative application of truncated Laplacian pyramids kernels for denoising signals by non-local means.
}

\section{Introduction}
\label{sec-intro}
This paper analyzes the Laplacian pyramids (LP) algorithm for extending a function sampled on a discrete set of points to outside values. This method was introduced in the context of machine learning by Rabin and Coifman in \cite{rabin2012heterogeneous}, and is modeled after the classical Laplacian pyramids algorithm of Burt and Adelson \cite{burt1983laplacian}, which is a standard technique in image processing. The LP extension algorithm has been considered in a variety of applications \cite{dsilva2013nonlinear, chiavazzo2014reduced, spingarn2014voice, aizenbud2015pca, mishne2014multiscale, comeau2019predicting, alexander2017kernel}, and several variants have been proposed \cite{fernandez2013auto, rabin2018multiscale, rabin2019two}.

The LP algorithm constructs a multiscale decomposition of the estimated function, consisting of averaged differences at successive levels. At each level, the residuals from the previous approximation are averaged and extended to the entire domain. The level $0$ approximation is just a weighted average of the observed values. At each sampled point, the residual is then computed, and the average residual is added to form the level $1$ approximation. The residuals are computed again, and the average residuals are again added back. This process is repeated, constructing a sequence of approximations at each successive level.

A key observation driving the extension method is that to compute the average residuals, we may use a kernel that is defined on points outside the samples. That is, while the residuals necessarily make use of the observed values, the \emph{averaged} residuals are well-defined everywhere, because the averaging kernel may be computed out-of-sample. Furthermore, a different averaging kernel may be used at each level. The sequence of bandwidths defining the extent of each kernel is typically chosen to be decreasing, with a large initial bandwidth to permit wide extrapolation.

In this paper, we prove certain properties about the LP extension method. First, we show that the scheme does in fact interpolate the observed values (to arbitrarily high precision), and show how the rate of convergence, i.e.\ the number of levels used in the extension scheme, is controlled by the choice of bandwidths. In particular, we show that the scheme may converge even when the kernel bandwidths do not shrink to $0$, or equivalently, when the averaging kernels do not converge to the identity matrix on the sampled points. This permits avoiding the use of small-bandwidth kernels which can introduce spurious artifacts into the extension.

Second, we show that for certain sequences of bandwidths the algorithm is stable, in the sense that the output function is bounded in terms of the maximum value of the input data. The stability bounds we derive are analogous to the stability bound from \cite{demarchi2010stability} for classical kernel interpolation methods that involve a single kernel at one scale. Our bound increases with the ratio of the maximum bandwidth to the minimum distance between the sample points, raised to a power that scales inversely to the rate of bandwidth decay.

Third, we consider the use of iterated truncated LP kernels for signal denoising by non-local means (NL means). The two-level version of a truncated LP kernel was employed in this fashion in \cite{singer2009diffusion}, and was shown to have advantages over a traditional NL means kernel. We consider the advantages of iterating higher-step kernels as well.

Our results are mainly derived from a simple formula for the residual terms of LP at each iteration. This formula expresses the residual operator at each level as a product of differencing operators from the previous scales. Similar decompositions have previously been observed for certain examples of boosting \cite{buhlmann2003boosting, milanfar2012tour, demarchi2010stability, buhlmann2006sparse, buhlmann2006boosting, buhlmann2010boosting, buhlmann2007boosting, marzio2008boosting}, though the applicability to the LP extension algorithm appears to be new.

The rest of the paper is organized as follows. In the remainder of Section \ref{sec-intro}, we review the LP extension algorithm, and compare it to other kernel-based methods for function extension. In Section \ref{sec-analysis}, we state and prove the main analytical results, namely the factorization of the residual operators, convergence, and stability. In Section \ref{sec-denoising}, we illustrate the use of truncated LP kernels for denoising by non-local means. In Section \ref{sec-conclusion} we provide a brief conclusion.

\subsection{The Laplacian pyramids algorithm}
\label{sec-lp}

In this section, we first review the method for Laplacian pyramids extension, as described in \cite{rabin2012heterogeneous}. We are given $n$ samples $\XX = \{ x_1, \dots, x_n \}$ from $\R^p$. For any point $x \in \R^p$, we are given a family of kernels $P_0(x,x_j), P_1(x,x_j), P_2(x,x_j)\dots$, defined on $\R^p \times \XX$, which capture the affinity between points $x \in \R^p$ to the sampled points $x_j \in \XX$. In this paper we will define the affinities by a radial kernel $\Phi(r)$; that is, we first define:
\begin{align}
G_\ell(w) = \Phi(\|w\| / \sigma_\ell),
\end{align}
for some bandwidth $\sigma_\ell > 0$, and then define the kernel $P_\ell$ by
\begin{align}
P_\ell (x,x_j) = \frac{G_\ell(x-x_j)}{\sum_{j'=1}^{n} G_\ell(x-x_{j'}) }.
\end{align}
For instance, the function $\Phi(r)$ may be taken to be a Gaussian, $\Phi(r) = e^{-r^2}$ (as suggested in \cite{rabin2012heterogeneous}, and frequently used in applications). In \cite{rabin2012heterogeneous} and most applications we have seen, the sequence of bandwidths $\sigma_\ell$ are taken to be geometrically decreasing; that is,
\begin{align}
\sigma_\ell = \sigma_0 / \mu^\ell, \quad \ell \ge 0,
\end{align}
for some value $\mu > 1$; $\mu = 2$ is a typical choice.

We are given the values $y_j = f(x_j)$ of a function $f$ at the points $x_j$. Given a new point $x \in \R^p$, the LP scheme extends $f$ to $x$ by defining a sequence of approximations as follows. The first approximation to $f(x)$ is defined as
\begin{align}
s_0(x) = \sum_{j=1}^{n} P_0(x,x_j) f(x_j).
\end{align}
If $P_0$ is row-stochastic over the $x_j$'s then $s_0(x)$ is a weighted average of the observed values $f(x_j)$. We will also denote this by $f_0(x) = s_0(x)$.

At the sample points $x_j$, $f_0(x_j)$ is an average over all the points $x_1,\dots,x_n$, and so generally will not be equal to $f(x_j)$. At each sample point $x_j$, we compute the residual term defined by
\begin{align}
d_1(x_j) = f(x_j) - f_0(x_j).
\end{align}
%
By definition, $f(x_j) = s_0(x_j) + d_1(x_j)$; so our next task is to extend $d_1$ to the out-of-sample point $x$. To extend $d_1$, we use the next kernel $P_1$, defining
\begin{align}
s_1(x) = \sum_{j=1}^{n} P_1(x,x_j) d_1(x_j).
\end{align}
We now can define the level $1$ approximation to $f(x)$ as the sum of $s_0(x)$ and $s_1(x)$, namely
\begin{align}
f_1(x) = s_0(x) + s_1(x).
\end{align}

The entire procedure may now be repeated again, at every level. We construct a sequence of estimators $f_\ell(x) = s_0(x) + \dots + s_\ell(x)$, where
\begin{align}
\label{eq-s_ell}
s_\ell(x) = \sum_{j=1}^{n} P_\ell(x,x_j) d_\ell(x_j),
\end{align}
and
\begin{align}
d_\ell(x_j) = f(x_j) - f_{\ell-1}(x_j) = f(x_j) - (s_0(x_j) + \dots s_{\ell-1}(x_j)).
\end{align}

In other words, starting with the level $\ell-1$ approximation, $f_{\ell-1}(x)$, we find its residuals $d_\ell(x_j)$ at the known points, and define $s_\ell$ by approximately extrapolating these residuals everywhere using kernel $P_\ell$, and then form our refined estimate $f_\ell$ by adding the estimated residual $s_\ell$ to $f_{\ell - 1}$.

\subsection{Other kernel-based methods}
\label{sec-other}

The LP algorithm is similar to other kernel-based methods for extending functions sampled on discrete points. We mention two approaches in particular. Kernel interpolation takes a fixed radial function $G(w)$, and seeks to approximate $f$ by writing
\begin{align}
f(x) = \sum_{i=1}^{n} \alpha_i G(x - x_i).
\end{align}
Because this expression is linear in the coeffcients $\alpha_i$, they may be fit by least-squares, to ensure that $f(x_i) = y_i$ on the sampled points $x_i$. 

One drawback of this class of methods is that they may suffer from numerical instabilities due to the fitting procedure. This is especially true if the kernel $G$ is chosen to have a large bandwidth, since in this case the functions $G(x - x_i)$ may be nearly linearly dependent if the $x_i$ are too close, and the resulting linear system for the $\alpha_i$ is ill-conditioned.

An alternative approach that is used primarily in the statistics community is known as the Nadaraya-Watson (NW) estimator \cite{watson1964smooth, nadaraya1963estimating}. This takes a kernel $G$, and writes the estimated function $f$ as the weighted average of observed values:
\begin{align}
f(x) = \frac{\sum_{i=1}^{n} G(x-x_i) y_i}{\sum_{i=1}^{n} G(x-x_i)}.
\end{align}

A modification of NW is proposed in \cite{marzio2008boosting} using the method of $L_2$ boosting \cite{buhlmann2003boosting}. The residuals at each level are fit using the same kernel $G$, and the process is then iterated several times. This method can be seen as a special case of LP, where the same bandwidth $\sigma_\ell$ is used at every scale, although there is no extra work in introducing variable bandwidths. In this sense, LP and NW with $L_2$ boosting are essentially identical methods.

\subsection{Notation}

We will denote by $\overline P_k$ the $n$-by-$n$ matrix with $(i,j)^{th}$ entry $P_k(x_i,x_j)$. The matrix $\overline P_k$ is the discretization of the kernel $P_k$ on the $n$ sampled points $x_1, \dots, x_n$.

Similarly, for any function $g$ defined on all of $\R^p$, we will denote by $\overline g$ the vector of samples:
\begin{align}
\overline g = (g(x_1), \dots, g(x_n))^T.
\end{align}

We will also define the following matrices. Let $A_\ell$ be the $\ell^{th}$ level LP operator, mapping the vector $\oline f$ of observed values to the $\ell^{th}$ level approximation $f_\ell$; that is, $f_\ell = A_\ell \oline f$. Following our previous notation, denote by $\oline A_\ell$ the $n$-by-$n$ matrix whose rows are restricted to $x_1,\dots,x_n$; in this notation, $\oline f_\ell = \oline A_\ell \oline f$.

Define $S_\ell$ to be the operator mapping $\oline f$ to $s_\ell$, defined by \eqref{eq-s_ell}; that is, $s_\ell = S_\ell \oline f$. Again, we will let $\oline S_\ell$ be the $n$-by-$n$ matrix whose rows are restricted to $x_1,\dots,x_n$.

Finally, we let $D_\ell$ denote the the differencing operator $I - \oline A_{\ell - 1}$, so that $d_\ell = \oline f - \oline f_{\ell - 1} = (I - \oline A_{\ell - 1}) \oline f = D_\ell \oline f$. Note that the differencing operators are only defined on the in-sample points $x_j$, which is why we do not use extra notation in this case.

\section{Analysis of LP: convergence and stability}
\label{sec-analysis}

In this section we will address several basic questions about the LP extension algorithm. First, it is not obvious under what conditions the scheme will converge to the observed values $y_j$ on the in-sample points $x_j$. At level $\ell$ the residual vectors $d_\ell$ are averaged using the kernel $P_\ell$, and these averaged residuals are added to the approximation. To guarantee convergence of $f_\ell(x_j)$ to $y_j$, one might suppose that at high levels the residuals $d_\ell$ must be approximated arbitrarily well -- that is, that the matrices $\oline P_\ell$ should approach the identity matrix, or equivalently that the bandwidths $\sigma_\ell$ approach $0$.

As we will show, it turns out that this is not necessary. The LP scheme will interpolate the given points so long as the $\oline P_\ell$ are sufficiently close to the identity; however, they do not need to approach the identity. In particular, the sequence of bandwidths may plateau at a sufficiently small value instead of approaching $0$ and the scheme will still converge. (The convergence rate, however, will depend on the decay of the bandwidths.) We will also demonstrate on a numerical example that there can be advantages to not using arbitrarily small bandwidths, as small-bandwidth kernels may introduce high-frequency artifacts into the extension.

We will also show that under the same conditions on the bandwidths, the LP algorithm is stable. More precisely, the infinity norm of the exended function cannot exceed a constant times the infinity norm of the input values. Phrased differently, treating LP as an operator that maps the input vector $y = (y_1,\dots,y_n)^T$ to the extended function $f_K$, we show that LP is a bounded operator from $\ell_\infty$ to $L^\infty$. The bound on the operator norm we derive exhibits a similar scaling as bounds for classical kernel interpolation methods shown in \cite{demarchi2010stability}.

\subsection{Factorization of the residual operators $D_k$}

This section derives a factorization of the residual operators $D_k$, which will be used repeatedly throughout the rest of paper. A similar formula has been shown for certain boosting methods in statistics; see \cite{buhlmann2003boosting, milanfar2012tour, demarchi2010stability, buhlmann2006sparse}. For completeness we provide a self-contained statement and derivation here.

\begin{prop}
\label{prop-factor}
The operators $D_\ell$ may be factored as follows:
\begin{align}
\label{eq-factor}
D_\ell = (I - \overline P_{\ell-1}) \cdots (I - \overline P_0),
\end{align}
for each $\ell \ge 1$.
\end{prop}
\begin{proof}
By definition, $A_0 = P_0$, and so $D_1 = I - \oline A_0 = I - \oline P_0$, proving the claim when $\ell=1$. We now proceed by induction. Suppose we have shown that $D_\ell = (I - \overline P_{\ell-1}) \cdots (I - \overline P_0)$ for some $\ell \ge 1$. Because $\oline A_{\ell-1} = I - D_\ell$ and $S_\ell = P_\ell D_\ell$, we have:
\begin{align}
\oline A_\ell = \oline A_{\ell-1} + \oline S_\ell = I - D_\ell + \oline P_\ell D_\ell
    = I - (I - \oline P_\ell) D_\ell
    = I - (I - \oline P_\ell)(I - \oline P_{\ell-1}) \cdots (I - \oline P_0)
\end{align}
and consequently
\begin{align}
D_{\ell + 1} = I - \oline A_\ell 
    = (I - \oline P_\ell)(I - \oline P_{\ell-1}) \cdots (I - \oline P_0),
\end{align}
proving the factorization formula for all $\ell$.
\end{proof}

\subsection{Convergence of LP}
\label{sec-convergence}

The factorization \eqref{eq-factor} of $D_\ell$ from Proposition \ref{prop-factor} has a trivial corollary, which implies convergence of the LP scheme (and bounds on its error) for a broad range of operators $P_\ell$.

\begin{prop}
The relative error of the $\ell^{th}$ level LP approximation on the $x_j$'s is bounded by:
\begin{align}
\frac{\| \overline f_\ell - \overline f\|}{\|\overline f\| } 
    \le \prod_{k=0}^{\ell-1} \|I - \overline P_k\|.
\end{align}
Here, $\|\cdot\|$ denotes any norm on $\R^p$ when applied to a vector, and the corresponding induced matrix norm when applied to a matrix.
\end{prop}

\begin{cor}
\label{cor-convergence}
If for some $0 < \epsilon < 1$ and $L \ge 1$ we have $\|I - \overline P_\ell\| \le \epsilon$ for $\ell > L$, then $\oline f_\ell \to \oline f$ as $\ell \to \infty$. In fact,
\begin{align}
\| \overline f_{\ell+1} - \overline f \| \le \epsilon \| \overline f_\ell - \overline f \|,
\end{align}
for all $\ell > L$.
\end{cor}

In particular, Corollary \ref{cor-convergence} shows that the $\oline P_\ell$ do not need to converge to the identity in order for LP to extend $\oline f$. It is enough that $\oline P_\ell$ be sufficiently close to $I$ in some norm.

We next show that when the bandwidth $\sigma_\ell$ is sufficiently small, the infinity norm of $I - \oline P_\ell$ is indeed less than $1$, allowing us to invoke Corollary \ref{cor-convergence} to show convergence. We define $\delta = \delta(\XX)$ to be the minimum Euclidean distance separating any two distinct points in $\XX$:
\begin{align}
\label{eq-delta}
\delta = \min_{1 \le i \ne j \le n} \|x_i - x_j\|_2.
\end{align}

We will assume that the radial kernel $\Phi(r)$ is decreasing as a function of $r\ge0$, and satisfies the following decay condition:
\begin{align}
\label{eq-decay}
\Phi(r) \le C r^{-q}, \quad r > 0,
\end{align}
for some parameter $q > p$ and constant $C > 0$. This family includes the Gaussian kernels (for any value of $q$). 

We then have the following result:

\begin{prop}
\label{prop-convergence}
Assume $\Phi(0) = 1$, $\Phi(r)$ decreases as a function of $r \ge 0$, and $\Phi$ satisfies condition \eqref{eq-decay}. Then for $0<\epsilon < 1$ there is a constant $c = c(p,\epsilon)$ such that 
\begin{align}
\|I - \oline P_\ell\|_\infty < \epsilon
\end{align}
if the bandwidth of $P_\ell$ satisfies $\sigma_\ell < c \delta$. In particular, if $\sigma_\ell < c \delta$ for all $\ell > L$, then $\oline f_\ell $ will converge to $\oline f$ as $\ell \to \infty$; in fact $\| \overline f_{\ell+1} - \overline f \| \le \epsilon \| \overline f_\ell - \overline f \|$ for $\ell > L$.
\end{prop}

\begin{proof}
The infinity norm of a matrix is the largest $\ell_1$ norm of its rows. Since $\oline P_\ell$ is row-stochastic, this implies
\begin{align}
\|I - \oline P_\ell\|_\infty = 2\max_{1 \le i \le n} (1 - \oline P_\ell(x_i,x_i)).
\end{align}
This is less than $\epsilon$ precisely when
\begin{align}
\sum_{j \ne i} G_\ell(x_i - x_j) \le \frac{\epsilon}{2-\epsilon} \equiv \eta
\end{align}
for all $i=1,\dots,n$.

Fix a value $i$. Because the $x_j$'s are all at least $\delta$ from each other, the number of points $N_r$ contained in any ball $B(x_i,r)$ cannot exceed $(2r/\delta + 1)^p$. Indeed, since $|B(x_i,r)| = C_p r^p$ and the balls $B(x_j,\delta/2)$ are disjoint, we have
\begin{align}
N_r C_p (\delta/2)^p \le C_p (r + \delta/2)^p.
\end{align}

Consequently, setting $R_k = B(x_i,2^{k+1}\delta) \setminus B(x_i,2^k\delta)$, we have the bound:
\begin{align}
\label{eq5756}
\sum_{j \ne i} G_\ell(x_i - x_j)
&= \sum_{k=0}^{\infty} \sum_{x_j \in R_k} G_\ell(x_i - x_j)
\le C \sum_{k=0}^{\infty} (2^{k+2} + 1)^p \Phi(2^k \delta / \sigma_\ell)
\nonumber \\
& \le C_p \left(\frac{\sigma_\ell}{\delta}\right)^q \sum_{k=0}^{\infty} 2^{kp} 2^{-kq}
= C_{p}  \left(\frac{\sigma_\ell}{\delta}\right)^q,
\end{align}
where $C_p$ denotes a constant depending on the dimension $p$ and the kernel $\Phi$. The expression on the right of \eqref{eq5756} will be less than $\eta$ whenever 
\begin{align}
\sigma_\ell \le (\eta / C_{p})^{1/q} \delta,
\end{align}
which is the desired result.
\end{proof}

\subsection{Stability of LP}
\label{sec-stability}

In this section, we will show that the LP scheme is stable, in the sense that the extended function $f_K$ can be bounded by the size of the input vectors $\oline f$. We will consider the same class of radial kernel $G_\ell(x-y) = \Phi(\|x-y\| / \sigma_\ell)$ considered in Section \ref{sec-convergence}, where $\Phi(r)$ is decreasing and satisfies the decay condition \eqref{eq-decay}.

Stability estimates like the ones we will prove have been shown previously for interpolating methods of the form
\begin{align}
f(x) = \sum_{j=1}^{n} \alpha_j \Phi(\|x-x_j\|),
\end{align}
where $\Phi$ satisfies a specified decay condition, and the coefficients $\alpha_j$ are found by least squares; see, for example, \cite{demarchi2010stability}. We note, however, that the condition imposed in \cite{demarchi2010stability} does not apply to as broad a family of kernels as we assume here, specifically Gaussian kernels.

We define $\delta$ to be the minimum distance between distinct points in $\XX$, as in \eqref{eq-delta}. We first prove a general estimate.

\begin{prop}
\label{prop-stability}
Suppose LP is performed with the sequence of bandwidths $\sigma_0,\sigma_1,\dots$. Take $\sigma^* < c\delta$, where $c = c(p,1/2)$ is the constant from Proposition \ref{prop-convergence}, and suppose for some $m$,
\begin{align}
\sigma_{j} \le \sigma^*, \quad j \ge m.
\end{align}
Then for all $\ell \ge 0$ we have the bound
\begin{align}
\|\oline f_\ell\|_\infty \le C 2^m \|\oline f\|_\infty
\end{align}
where $C$ is a universal constant.
\end{prop}

Taking a geometrically-decaying sequence of bandwidths, we immediately obtain the following corollary:

\begin{cor}
Suppose $\sigma_\ell = \sigma_0 / \mu^\ell$, where $\mu > 1$. Then for all $\ell \ge 0$ we have the bound
\begin{align}
\|\oline f_\ell\|_\infty 
\le C_p \left(\frac{\sigma_0}{\delta}\right)^{t} \|\oline f\|_\infty
\end{align}
where $C_p$ is a constant depending on the dimension $p$ and the kernel $\Phi$, and where $t = \log_\mu(2)$. The same estimate also holds if $\sigma_\ell = \max\{\sigma_0 / \mu^\ell, \sigma^*\}$, where $\sigma^* < c \delta$.
\end{cor}

\begin{proof}[Proof of Proposition \ref{prop-stability}]
We write the expansion
\begin{math}
f_\ell = \sum_{k=0}^{\ell} P_k D_k \oline f.
\end{math}
Since $\sum_{j=1}^{n} P_k(x,x_j) = 1$, it follows that
\begin{align}
\label{eq-split12345}
\|f_\ell\|_\infty \le \sum_{k=0}^{\ell} \| D_k \oline f \|_\infty
    = \sum_{k=0}^{m} \| D_k \oline f \|_\infty
        + \sum_{k=m+1}^{\ell} \| D_k \oline f \|_\infty.
\end{align}

We bound the first term:
\begin{align}
\label{ineq2000}
\sum_{k=0}^{m} \| D_k \oline f \|_\infty
    \le 2^{m+1} \| \oline f\|_\infty.
\end{align}
Indeed, for any vector $v \in \R^p$,
\begin{math}
\| (I - \oline P_\ell) v \|_\infty \le 2 \|v\|_\infty.
\end{math}
Consequently,
\begin{math}
\|D_k \oline f\|_\infty \le 2^{k} \|\oline f\|_\infty,
\end{math}
and summing a geometric series we obtain \eqref{ineq2000}.

Next, for any $k \ge 0$, the choice of $\sigma^*$ and Proposition \ref{prop-convergence} tells us that $\| I - \oline P_{m+k}\|_\infty \le 1/2$. Consequently, for any $j > 0$ we have:
\begin{align}
\|D_{m+j} \oline f\|_\infty 
    \le \|D_m \oline f\|_\infty \prod_{k=0}^{j-1} \|I - \oline P_{m+k}\|_\infty
\le \|D_m \oline f\|_\infty 2^{-j}.
\end{align}
From summing a geometric series we then obtain the bound
\begin{align}
\label{ineq2500}
\sum_{k=m+1}^{\ell} \| D_k \oline f \|_\infty 
\le \|D_{m} \oline  f\|_\infty
\le 2^m \|\oline f\|_\infty.
\end{align}

Combining \eqref{eq-split12345}, \eqref{ineq2000} and \eqref{ineq2500} yields the result.
\end{proof}

\subsection{Example: interpolation on the circle}

In this section we demonstrate on a numerical example how LP can result in qualitatively different extensions depending on the choice of bandwidth sequence. In particular, kernels with small bandwidths are close to the identity on the sampled points $x_j$, and so can introduce high-frequency components into the extension not present in the original data, even when they perfectly interpolate the observed values.

%
%
\begin{figure}[h]
\center
\includegraphics[scale=.4]{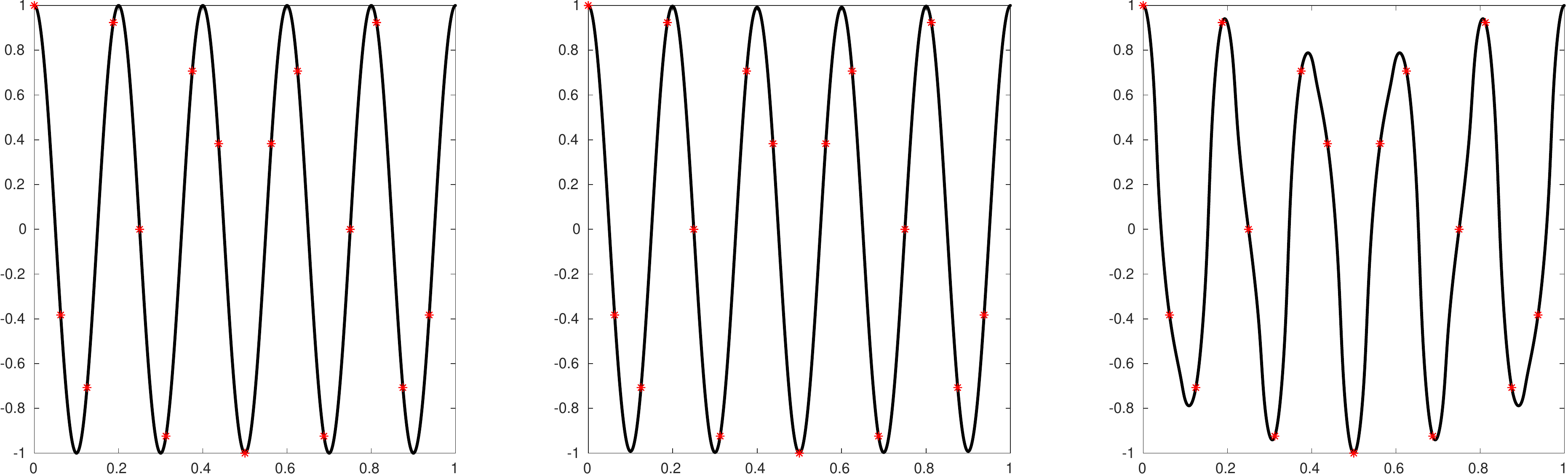}
\caption{Left: The function $f(x) = \cos(10 \pi x)$. Middle: The LP reconstruction with bandwidths $\sigma_\ell = \max\{2^{-\ell+1},1/2\}$. Right: The LP extension with bandwidths $\sigma_\ell = 2^{-\ell+1}$. On all figures the sampled points are highlighted in red.}
\label{fig-cosine}
\end{figure}

To illustrate this phenomenon, we sample $n = 16$ equispaced points $x_k = k/n$ from the circle $S^1 \subset \R^2$ of circumference $1$, and evaluate the function $f(x) = \cos(10 \pi x)$. In this case, there is enough information from the samples to perfectly interpolate $f$ on the entire circle. We plot the function $f$ in the left panel of Figure \ref{fig-cosine}, along with the sampled values.

In the right panel of Figure \ref{fig-cosine}, we plot the LP extension of $f$ using the geometrically-decreasing sequence of bandwidths $\sigma_\ell = 2^{-\ell+1}$, $\ell \ge 0$. In the middle panel of Figure \ref{fig-cosine}, we plot the LP extension of $f$ using the sequence of bandwidths $\sigma_\ell = \max\{2^{-\ell+1},1/2\}$, $\ell \ge 0$; in other words, the bandwidths plateau after the third scale. The relative errors are, respectively, $2.14 \times 10^{-1}$ and $6.14 \times 10^{-3}$. The geometrically-decreasing sequence requires only 6 levels until convergence to machine precision (approximately $10^{-14}$ in this case) on the sampled values, whereas the plateaued sequence requires 136 levels until convergence.

The reason for the higher error in the first scheme is that kernels with smaller bandwidth put more weight on the higher frequencies. In other words, these kernels introduce greater aliasing into the reconstruction. By choosing the plateaued sequence of bandwidths, we are able to mitigate the aliasing, at the expense of introducing more levels into the reconstruction.

\subsection{Example: extrapolation from an interval}

In this example, we illustrate the stability estimate from Proposition \ref{prop-stability} on an example. We take $n=16$ equispaced points on the interval $[0,1]$, and assign them alternating values $\pm 1$, so that $y_0=1$, $y_1=-1$, and so forth. We apply the LP extension procedure for geometrically decreasing bandwidths, $\sigma_\ell = \sigma_0 / \mu^\ell$. We plot an example of the extrapolated function, for $\mu=2$ and $\sigma_0=1$, in Figure \ref{fig-extrap}.

%
%
\begin{figure}[h]
    \centering
    \begin{subfigure}[t]{0.4\textwidth}
        \centering
        \includegraphics[scale=.4]{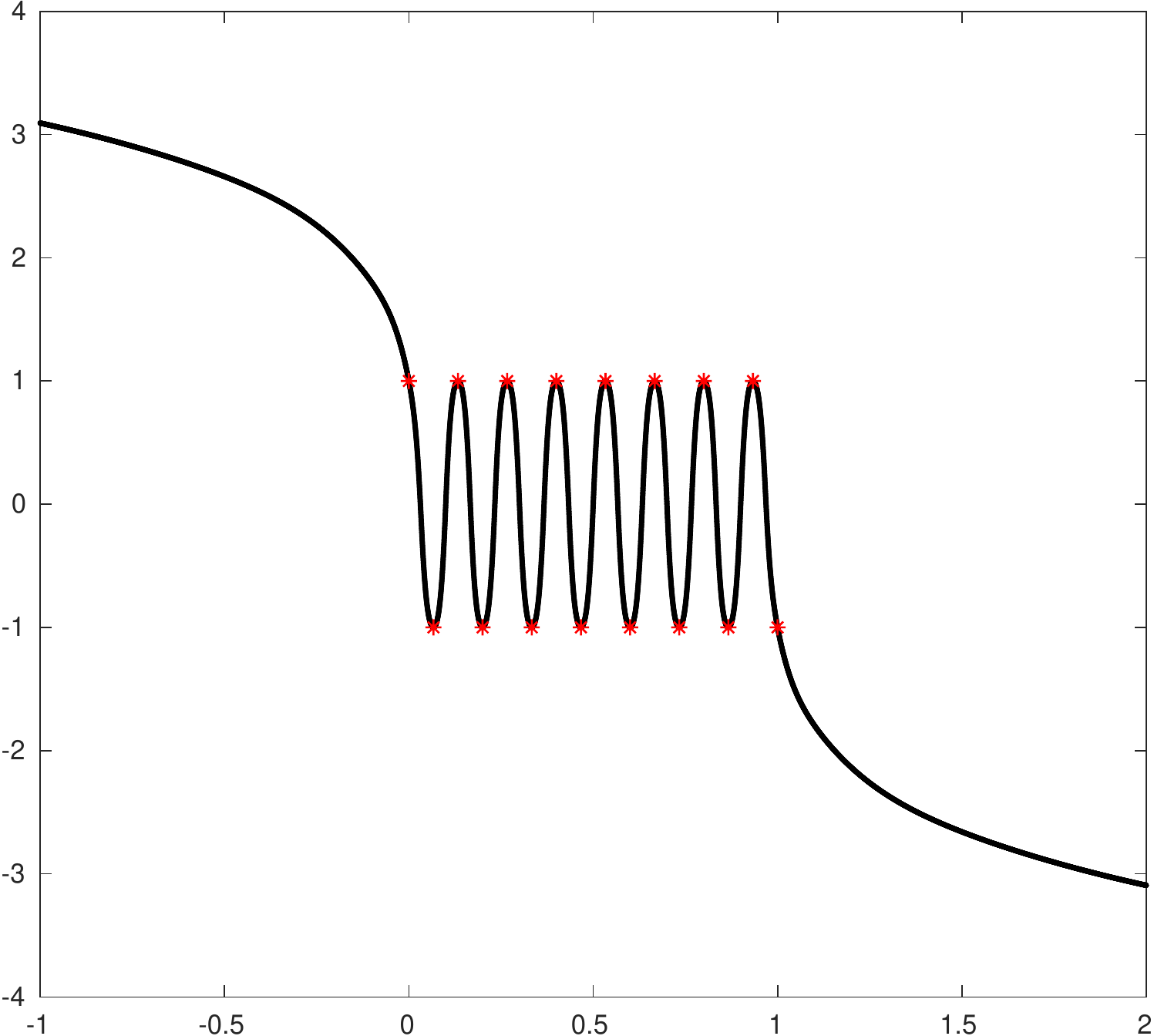}
    \end{subfigure}%
    \begin{subfigure}[t]{0.4\textwidth}
        \centering
        \includegraphics[scale=.4]{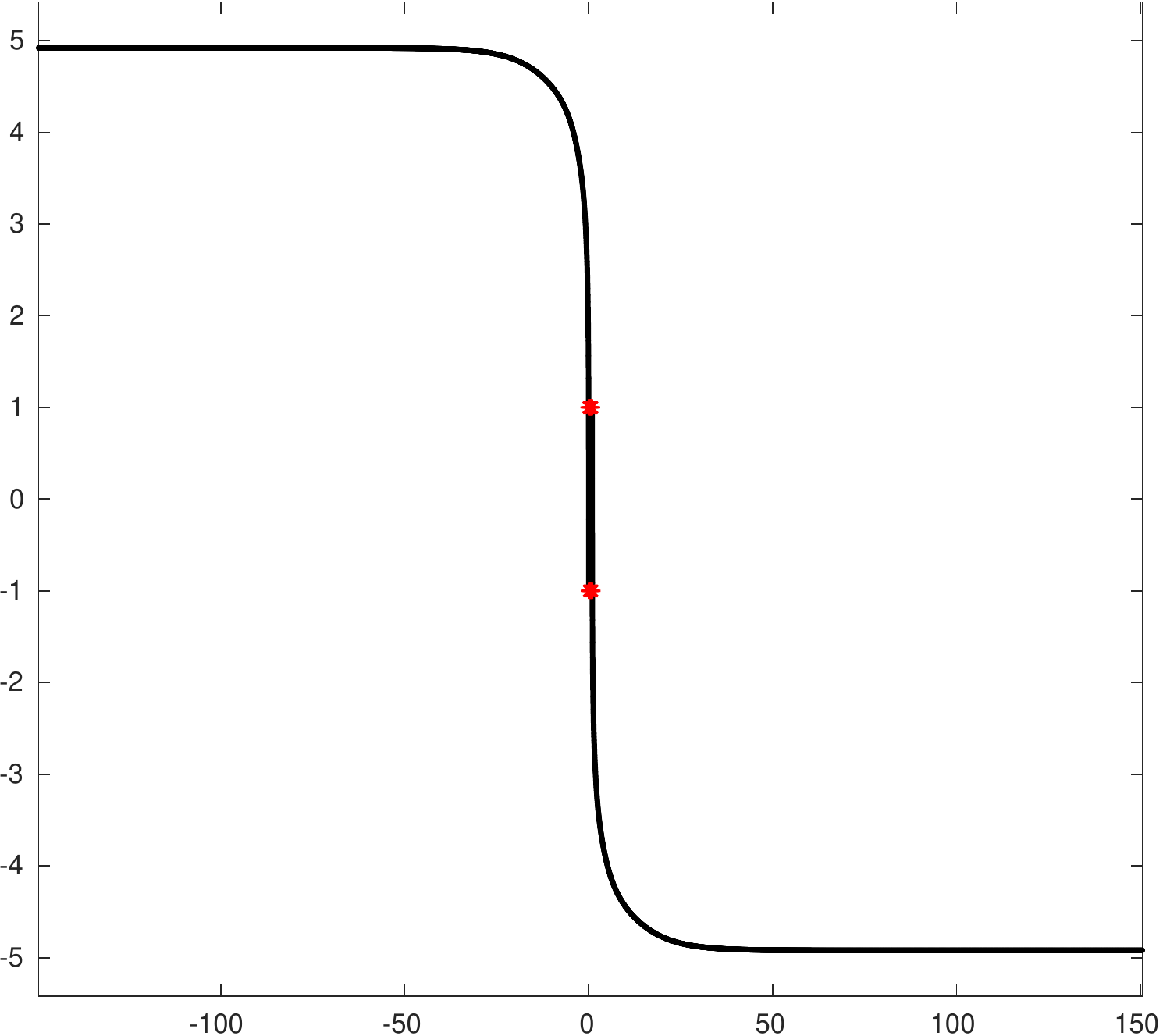}
    \end{subfigure}
    \caption{The extrapolated function, displayed at different scales. The observed values are highlighted in red.}
\label{fig-extrap}
\end{figure}

We are interested in exploring the size of the extrapolation as functions of the parameters $\sigma_0$ and $\mu$. Proposition \ref{prop-stability} predicts that larger values of $\sigma_0$ and smaller values of $\mu$ will result in larger extrapolated values. In the left panel of Figure \ref{fig-max}, we plot the maximum value (to within precision $10^{-7}$) of the extrapolated function as a function of $\sigma_0$. Indeed, we see that as $\sigma_0$ grows, the infinity norm of the extrapolation increases.

Similarly, in the right panel of Figure \ref{fig-max} we plot the maximum value (to within precision $10^{-7}$) of the extrapolated function as a function of the decay rate $\mu$. The infinity norm of the extrapolation increases with decreasing $\mu$. Again, this is the qualitative behavior expected from Proposition \ref{prop-stability}.

%
%
\begin{figure}[h]
    \centering
    \begin{subfigure}[t]{0.4\textwidth}
        \centering
        \includegraphics[scale=.5]{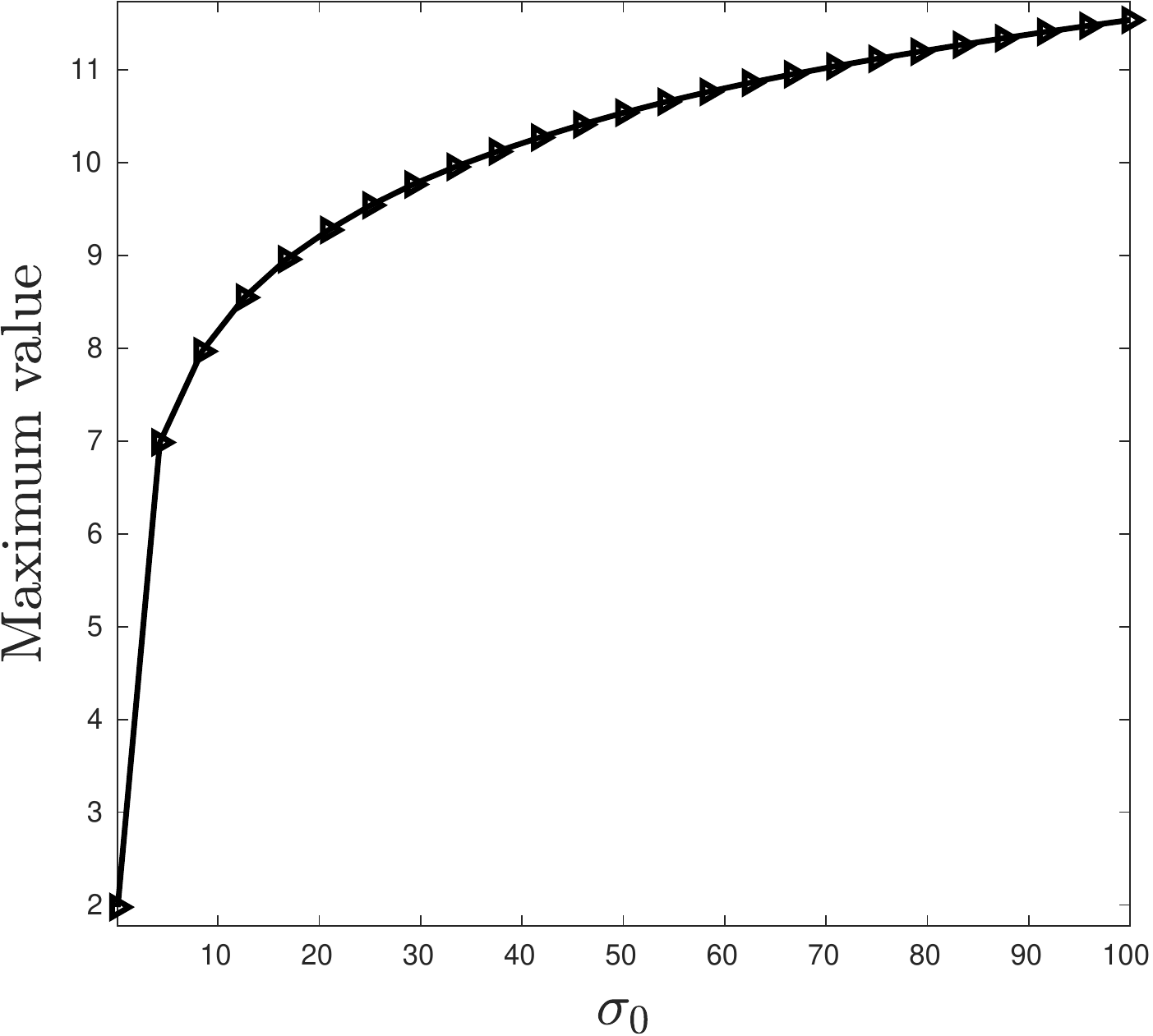}
    \end{subfigure}%
    \begin{subfigure}[t]{0.4\textwidth}
        \centering
        \includegraphics[scale=.5]{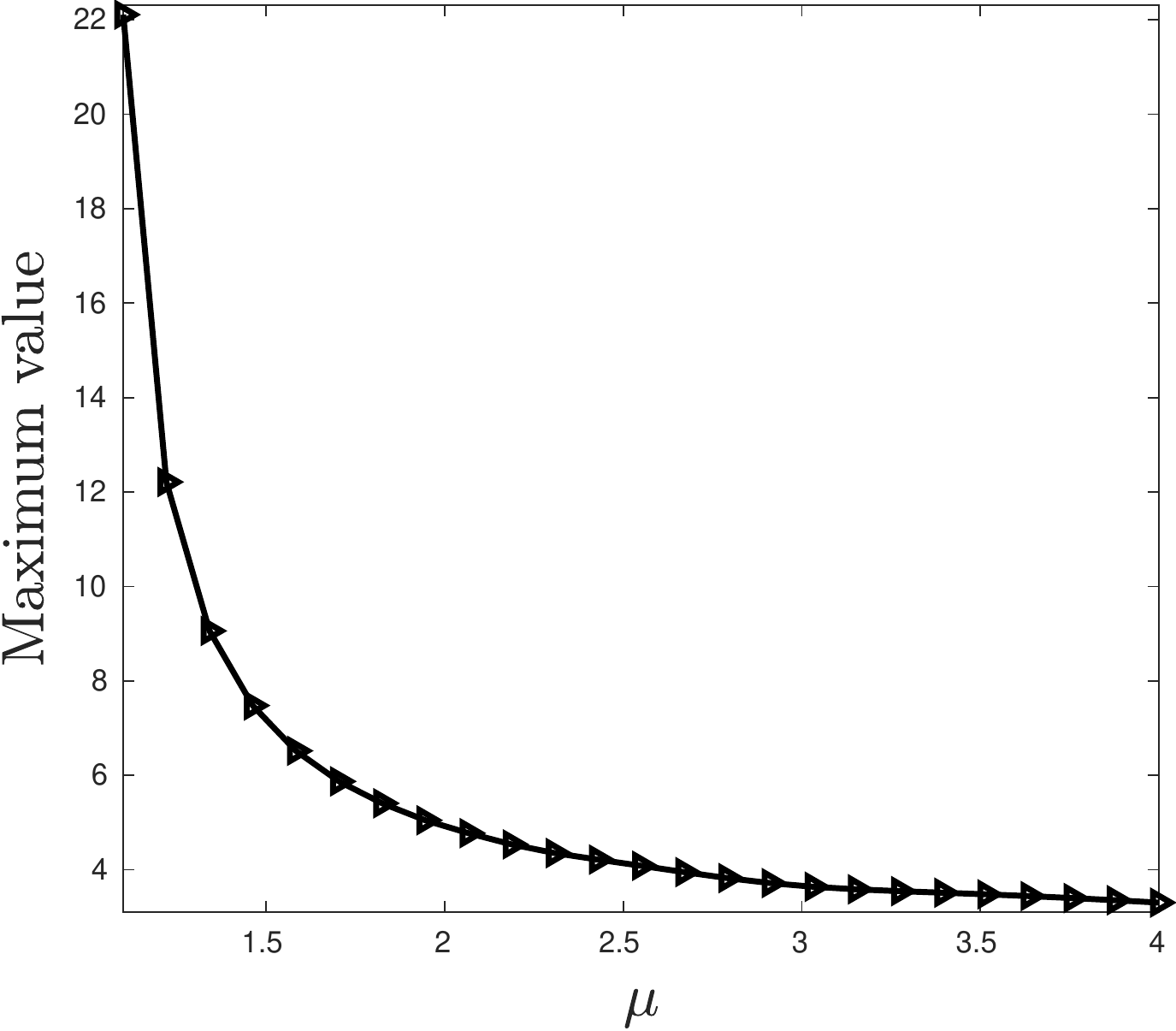}
    \end{subfigure}
    \caption{Left: The maximum value of the extrapolation as a function of $\sigma_0$. Right: The maximum value of the extrapolation as a function of $\mu$.}
\label{fig-max}
\end{figure}

\section{Laplacian pyramids and denoising}
\label{sec-denoising}

In this section, we consider the problem of denoising the in-sample observations, rather than extending the function to new values. In \cite{fernandez2013auto}, it is proposed that when the observed data is noisy, the LP algorithm should be truncated before convergence to avoid overfitting; a method that approximates cross-validation is used to determine the stopping level. If $K$ levels are used, then from Proposition \ref{prop-factor}, the denoised vector is $\oline A_K y$, where $y$ is the observed vector and 
\begin{align}
\oline A_K = I - (I - \oline P_{K}) \cdots (I - \oline P_{0}).
\end{align}
If each $\oline P_\ell$ is row-stochastic, then so too is the denoising kernel $\oline A_K$.

In this special case where $\oline P_K = \dots = \oline P_0 = Q$, the denoising kernel $\oline A_K$ takes on a particularly simple form. Changing notation to $Q_K \equiv \oline A_K$, we have
\begin{align}
Q_K = I - (I - Q)^K.
\end{align}
(Note that $Q = Q_1$.) As we noted in Section \ref{sec-other}, $Q_K$ is the same kernel used when applying $L_2$ boosting to kernel regression, as described in \cite{marzio2008boosting}. In this section, we will consider using the kernels $Q_K$ in the context of non-local means denoising \cite{buades2005review, buades2005nonlocal}. We will first review the basic non-local means algorithm, and then compare the use of the iterated kernels $Q_K$ within the non-local means framework.

\subsection{Non-local means}
Given a signal $s \in \R^M$, we suppose that we observe $s$ in the presence of noise:
\begin{align}
y = s + \ep
\end{align}
where the entries of $\ep$ are noise, e.g.\ shot noise or Gaussian. NL means (in its simplest incarnation) performs the following procedure to remove the noise $\ep$. First, patches of adjacent samples (or pixels, in the case of an image) are extracted from the long signal; we call these vectors $x_1,\dots,x_n$. We will suppose each $x_i \in \R^m$, where $m \ll M$.

Second, an affinity between the patches $x_i$ is defined. For concreteness, we will use the common choice of a Gaussian kernel to specify the affinity, writing
\begin{align}
G(x_i,x_j) = \exp\{-\|x_i - x_j\|^2 / \sigma^2\}
\end{align}
where $\sigma > 0$ is a specified parameter.

Third, the affinities $G(x_i,x_j)$ are normalized to form the row-stochastic matrix $Q$:
\begin{align}
Q(x_i,x_j) = \frac{G(x_i,x_j)}{\sum_{j'} G(x_i,x_{j'})}.
\end{align}

With this Markov kernel now defined, one iteration of NL means is performed by taking
\begin{align}
s_{NL}^{(1)} = Q y.
\end{align}
In words, each entry of $y$ is replaced by a weighted average of the other entries, where the weights are determined by local patches.

Of course, this process can be iterated multiple times by repeated application of $Q$. In this way, we obtain a sequence of denoised images:
\begin{align}
s_{NL}^{(\ell)} = Q^\ell y.
\end{align}

A physical interpretation of this algorithm is provided in \cite{singer2009diffusion} . $s_{NL}^{(\ell)}[i]$ is equal to the expected value of a random process that takes $\ell$ steps along the patches $x_j$ starting at patch $x_i$, with transition probabilities specified by $Q$, where the value of the process at patch $x_j$ is $y_j$.

\subsection{The choice of kernel}

The transition probabilities along the patches $x_j$ can be specified by any Markov matrix, not just the local diffusion matrix $Q$. In particular, \cite{singer2009diffusion} proposes the alternative matrix
\begin{align}
\label{eq-A2}
Q_2 = 2 Q - Q^2 = I - (I - Q)^2.
\end{align}
As we have seen, the kernel $Q_2$ is equal to a two-step truncated LP kernel, or equivalently a two-step $L_2$ boosting kernel \cite{buhlmann2003boosting, marzio2008boosting}. As has been observed previously in \cite{buhlmann2003boosting, milanfar2012tour, marzio2008boosting}, $Q_2$ is also equal to the ``twicing'' kernel introduced by Tukey \cite{tukey1977exploratory}. It is illustrated in \cite{singer2009diffusion} on several examples that iteratively applying $Q_2$ may achieve better denoising than iteratively applying the original kernel $Q_1 \equiv Q$.

%

%
%
\begin{figure}[h]
\center
\includegraphics[scale=.4]{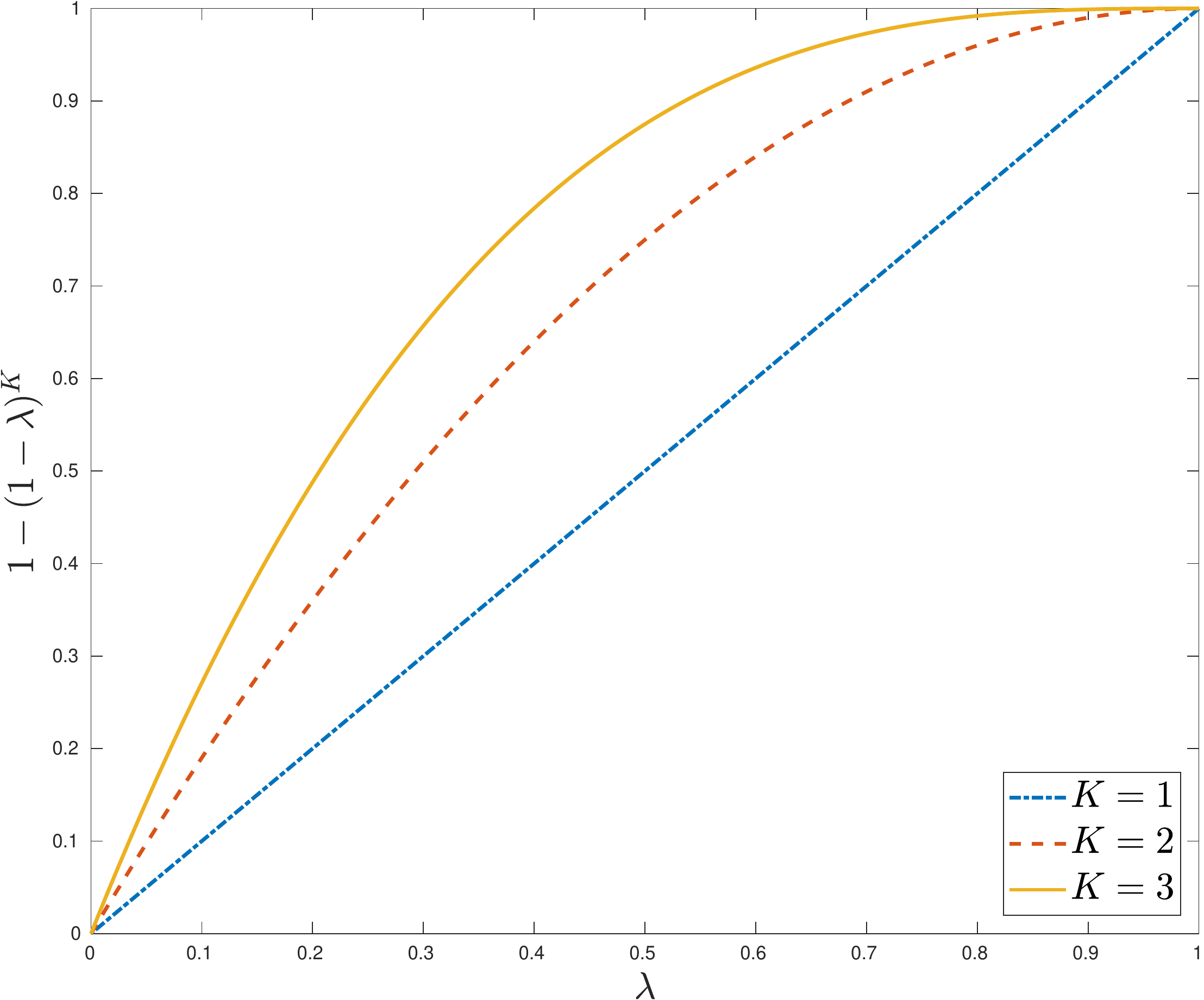}
\caption{The eigenvalues of $Q_K$ as functions of the eigenvalues of $Q$.}
\label{fig-curves}
\end{figure}

Of course, one may also consider running NL means by iteratively applying the higher-step LP kernels $Q_K = I - (I - Q)^K$ as well. Because $Q$ is diagonalizable with eigenvalues contained between $0$ and $1$, the truncated LP kernels $Q_K$ are also row stochastic, with eigenvalues
\begin{align}
1 - (1-\lambda)^K, \quad \lambda \in \text{spec}(Q).
\end{align}

In Figure \ref{fig-curves}, we plot the functions $1 - (1-\lambda)^K$ for several values of $K$. Larger values of $K$ result in kernels $Q_K$ closer to the identity $I$. Consequently, the iterations of NL means will converge more slowly to 0, allowing a more refined denoising procedure.

\subsection{Example: step function}

We illustrate the behavior of NL means with different kernels for denoising a 1D signal. The signal $s$, which was considered in \cite{singer2009diffusion}, has length $M=100$, which assumes two values, $-1$ and $+1$, and is observed with additive Gaussian noise of standard deviation $0.5$. The signal is plotted in the left side of Figure \ref{fig-step}, and the signal with noise is plotted in the right side.

%
%
\begin{figure}[h]
\center
\includegraphics[scale=.5]{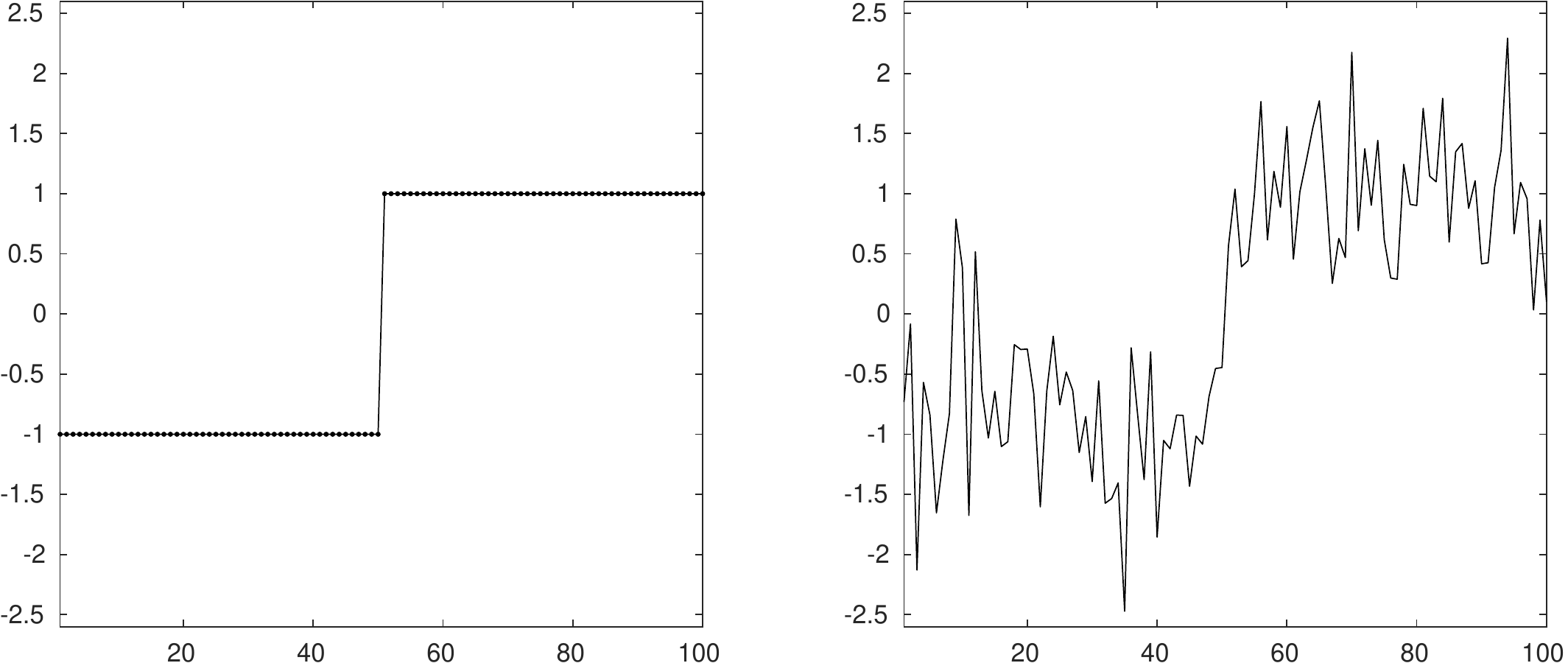}
\caption{Left: The clean step function. Right: The step function with Gaussian noise with standard deviation $0.5$.}
\label{fig-step}
\end{figure}

%
%
\begin{figure}[h]
\centering
\includegraphics[scale=.35]{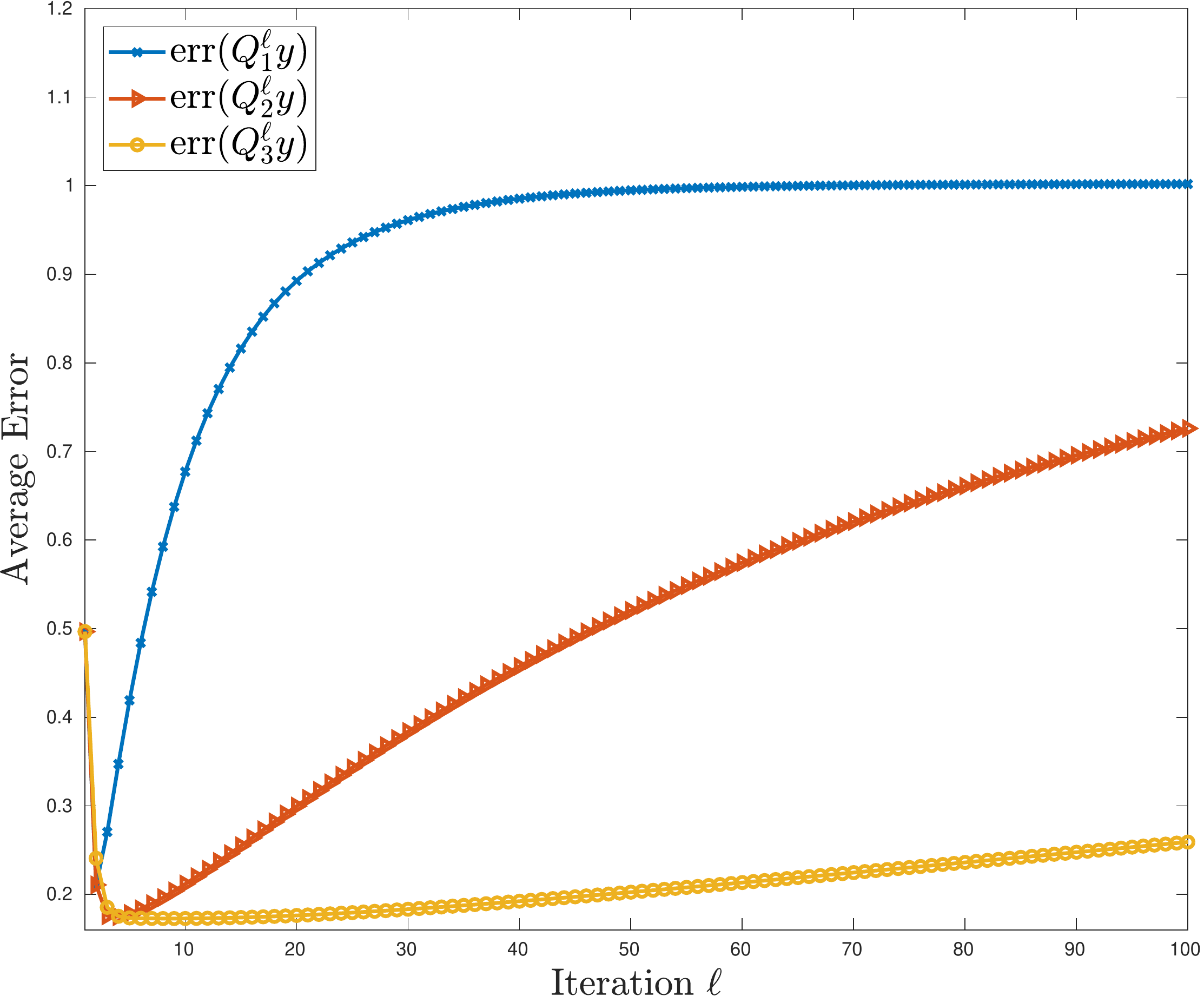}
\caption{The average errors when denoising by kernels $Q_1^\ell$, $Q_2^\ell$, and $Q_3^\ell$, as a function of the iteration $\ell$. Errors are averaged over 500 runs.}
\label{fig-errs}
\end{figure}

We build the NL means kernel using subintervals of size $m=3$. The Gaussian kernel matrix for the step function is built using the median squared distance between all pairs of points divided by $3$. This scaling value is somewhat arbitrary, and was manually chosen to ensure that the graph defined by $Q$ is not too connected.

In Figure \ref{fig-errs}, we plot the errors of NL means as a function of the number of iterations, for kernels $Q_K$ with parameters $K=1,2,3$. That is, we plot:
\begin{align}
\text{err}(Q_K^\ell y) = \frac{\|Q_K^\ell y - s\|_2}{\|s\|_2},
\end{align}
as a function of the iteration number $\ell$, where $s$ is the step function. For comparison, we also plot the average errors $\text{err}(Q_\ell y) = \|Q_\ell y - s\|_2 / \|s\|_2$ of applying the non-iterated LP kernels $Q_\ell$, as a function of the level $\ell$; these are the iterates we obtain by boosting. We emphasize that the scheme we propose uses a fixed value of $K$, and iteratively applies $Q_K$; the resulting denoising kernel is then $Q_K^\ell$, where $\ell$ is the number of iterates. The curves displayed are averaged over 500 runs of the experiment, where each experiment is run with a different realization of the noise. The minimal errors of $Q_2^\ell$ and $Q_3^\ell$ (over $\ell$) are both smaller than the minimal error for $Q_1^\ell$. The average minimal error for $Q_1^\ell$ is $0.210$, while they are $0.171$ and $0.168$ for $Q_2^\ell$ and $Q_3^\ell$, respectively. 

For any choice of Markov kernel, the iterations of NL means will both average out the noise and the signal. While the effect of the noise will be reduced, it will also result in smoothing of the signal by shrinking all the values towards the mean. In other words, increasing the iterations will increase the bias and decrease the variance. In general, given only the noisy signal $y$, it may be difficult to estimate the optimal number of iterations that minimizes the overall error.

In light of these considerations, while the minimal errors achieved by $Q_2^\ell$ and $Q_3^\ell$ are nearly identical, more interesting is that, because it takes longer for the spectrum of $Q_3^\ell$ to decay, there is a much larger range of iterations for which it does not yet oversmooth the signal, and hence where the error is smaller than the error for $Q_2^\ell$. In this sense, the sequence $Q_3^\ell$ is less sensitive to the number of iterations $\ell$ chosen by the user, and hence more rubust to the specification of this parameter.


%

\section{Conclusion}
\label{sec-conclusion}

We have proven several properties of the Laplacian pyramids extension algorithm. Based on the factorization formula from Proposition \ref{prop-factor}, we showed that the method always converges to an interpolator of the observed data if the kernel bandwidths drop below a certain threshold. We also proved a stability estimate for the extension, which exhibits similar qualitative behavior as prior estimates from \cite{demarchi2010stability} for classical kernel interpolation methods.

We also considered iterating the truncated LP kernels to denoise signals by non-local means. A scheme of this kind for a two-step kernel was proposed in \cite{singer2009diffusion}. Here, we have shown on numerical examples that using higher-step kernels may be advantageous, as they are less sensitive to the number of iterations chosen by the user, and may also achieve lower error overall with an optimal number of iterations. In future work, we plan to further explore the properties and behavior of these denoising kernels.

\section*{Acknowledgements}
I acknowledge support from the NSF BIGDATA program, IIS 1837992.

\bibliographystyle{plain}
\bibliography{refs}

\end{document}